\documentclass[10pt,twocolumn,letterpaper]{article}

\usepackage[pagenumbers]{cvpr} 

\usepackage{graphicx}
\usepackage{amsmath}
\usepackage{amssymb}
\usepackage{booktabs}
\usepackage[accsupp]{axessibility}

\usepackage{graphicx}
\usepackage{amsthm}
\usepackage{amsmath,amssymb} 
\usepackage{makecell}
\usepackage{color}
\usepackage{enumitem}
\usepackage{multirow}
\usepackage{booktabs}
\usepackage{bbm}
\usepackage{arydshln}
\usepackage{subcaption}
\usepackage{bm}
\usepackage{algorithm,algpseudocode}
\usepackage[accsupp]{axessibility}

\newtheorem{proposition}{Proposition}
\newcommand{\mypm}{\,$\pm$\,}

\newcommand{\red}[1]{\textbf{\textcolor{red}{#1}}}
\newcommand{\black}[1]{\textbf{\textcolor{black}{#1}}}

\usepackage[pagebackref,breaklinks,colorlinks]{hyperref}

\usepackage[capitalize]{cleveref}
\crefname{section}{Sec.}{Secs.}
\Crefname{section}{Section}{Sections}
\Crefname{table}{Table}{Tables}
\crefname{table}{Tab.}{Tabs.}

\def\cvprPaperID{5800} 
\def\confName{CVPR}
\def\confYear{2022}

\begin{document}

\title{Class-Incremental Learning by Knowledge Distillation \\
with Adaptive Feature Consolidation}  
\author{Minsoo Kang$^\dagger$ \qquad\qquad Jaeyoo Park$^\dagger$ \qquad\qquad Bohyung Han$^{\dagger, \S}$ \\
ECE$^\dagger$, ASRI$^\dagger$, \& IPAI$^{\dagger, \S}$ \\ Seoul National University\\
{\tt\small \{kminsoo,bellos1203,bhhan\}@snu.ac.kr}
}
\maketitle


\begin{abstract}
We present a novel class incremental learning approach based on deep neural networks, which continually learns new tasks with limited memory for storing examples in the previous tasks.
Our algorithm is based on knowledge distillation and provides a principled way to maintain the representations of old models while adjusting to new tasks effectively.
The proposed method estimates the relationship between the representation changes and the resulting loss increases incurred by model updates.
It minimizes the upper bound of the loss increases using the representations, which exploits the estimated importance of each feature map within a backbone model.
Based on the importance, the model restricts updates of important features for robustness while allowing changes in less critical features for flexibility. 
This optimization strategy effectively alleviates the notorious catastrophic forgetting problem despite the limited accessibility of data in the previous tasks.
The experimental results show significant accuracy improvement of the proposed algorithm over the existing methods on the standard datasets.
Code is available.\footnote{\url{https://github.com/kminsoo/AFC}}
\end{abstract}


\section{Introduction}
\label{sec:introduction}
Deep neural networks have achieved outstanding results in various applications including computer vision~\cite{he2016deep, nam2016learning, noh2015learning, redmon2016you, long2015fully}, natural language processing~\cite{NIPS2014_a14ac55a, vaswani2017attention}, speech recognition~\cite{chan2016listen,amodei2016deep}, robotics~\cite{lenz2015deep}, bioinformatics~\cite{angermueller2016deep}, and many others. 
Despite their impressive performance on offline learning problems, it is still challenging to train models for a sequence of tasks in an online manner, where only a limited number of examples  in the previous tasks are available due to memory constraints or privacy issues.

\begin{figure}[ht!]
    \centering
    \includegraphics[width=0.95\linewidth]{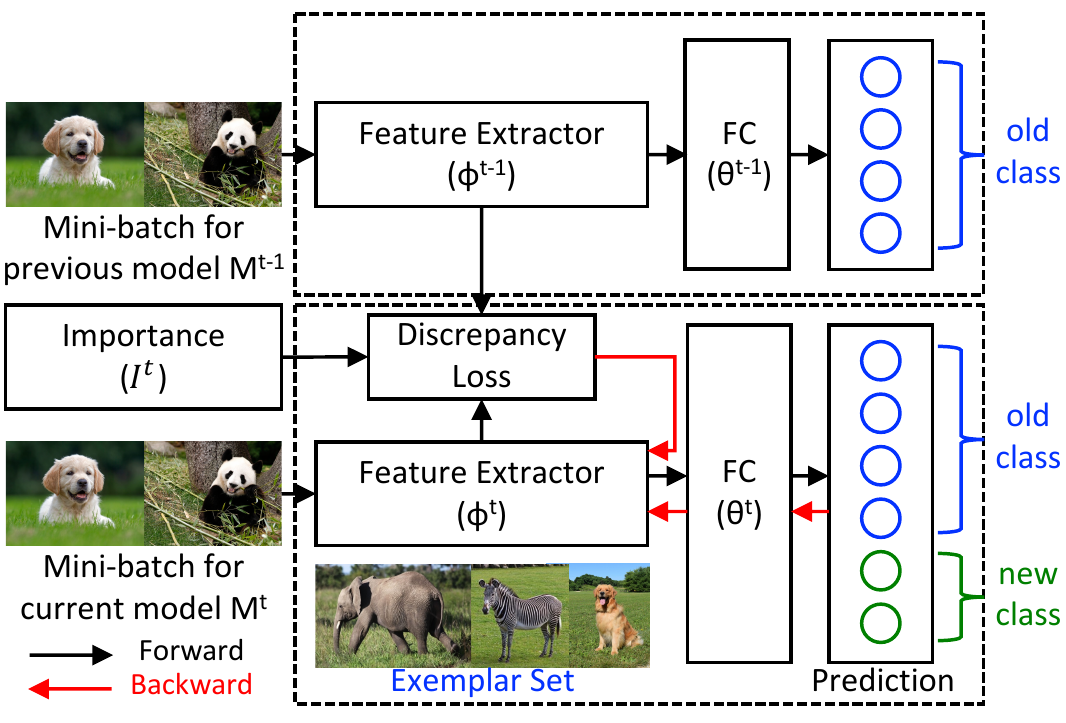}
    \caption{Illustration of the proposed method. The current model optimizes the task loss over a mini-batch sampled from the data of the current task and exemplar sets for the previous tasks while it aims to differently minimize the discrepancy in each feature map over the previous model based on the corresponding importance.
    } 
  \vspace{-0.2cm}
  \label{fig:importance}
\end{figure}  
Although fine-tuning is a good strategy to learn a new task given an old model, it is not effective for streaming tasks due to the catastrophic forgetting problem~\cite{mccloskey1989catastrophic}; the model performs well on the current task while it often fails to generalize on the previous ones. 
Therefore, incremental learning, the framework capacitating online learning without forgetting, has been studied actively.
To mitigate the catastrophic forgetting problem, several branches of algorithms have been discussed.
Architectural approaches employ network expansion schemes~\cite{rusu2016progressive, liu2021adaptive, yan2021dynamically} instead of using static models.
Rehearsal methods present the strategies to summarize the tasks in the past by storing exemplar sets~\cite{rebuffi2017icarl, aljundi2019gradient} or generate samples by estimating the data distribution in the previous tasks~\cite{shin2017continual}.
On the other hand, parameter regularization methods~\cite{kirkpatrick2017overcoming, zenke2017continual, aljundi2018memory} prevent important weights from deviating the previously learned models.
Knowledge distillation approaches~\cite{jung2016less, li2017learning} focus on minimizing the divergences from the representations of the previous models while effectively adapting to new tasks.

We present a knowledge distillation approach for class incremental learning, where the model incrementally learns new classes via knowledge distillation, with limited accessibility to the data from the classes in the old tasks.
Existing approaches based on knowledge distillation simply minimize the distances between the representations of the old and new models without considering which feature maps are important to maintain the previously acquired knowledge.
Although \cite{park2021class} addresses this issue, it is limited to the heuristic assignments of the importance weights.
On the contrary, we estimate how the representation change given by a model update affects the loss, and show that the minimization of the loss increases can be achieved by a proper weighting of feature maps for knowledge distillation.
In a nutshell, the proposed approach maintains important features for robustness, while making less critical features flexible for adaptivity.
Note that our optimization strategy aims to minimize the upper bound of the expected loss increases over the previous tasks, which eventually reduces the catastrophic forgetting problem.  
Figure~\ref{fig:importance} illustrates the main idea of the proposed approach.
The main contributions and characteristics of our algorithm are summarized below:
\begin{itemize}[label=$\bullet$]
	\item
	We propose a simple but effective knowledge distillation algorithm for class incremental learning based on knowledge distillation with feature map weighting.
	\item
	We theoretically show that the proposed approach minimizes the upper bound of the loss increases over the previous tasks, which is derived by recognizing the relationship between the distribution shift in a feature map and the change in the loss.
	\item
	Experimental results demonstrate that the proposed technique outperforms the existing methods by large margins in various scenarios.
\end{itemize}

The rest of the paper is organized as follows.
Section~\ref{sec:related} discusses related works to class incremental learning.
The details of our method is described in Section~\ref{sec:framework}, and the experimental results are presented in Section~\ref{sec:experiments}.
Finally, we conclude this paper in Section~\ref{sec:conclusion}.


\section{Related Work}
\label{sec:related}

Existing class incremental learning approaches are categorized into one of the following four groups: architectural, rehearsal, parameter regularization, and knowledge distillation methods.
This section reviews the previous methods in each category.

\subsection{Architectural Methods}
\label{sec:architecture}
Most architectural methods adjust network capacity dynamically to handle a sequence of incoming tasks.
Progressive networks~\cite{rusu2016progressive} augment a network component whenever learning  a new task.
Continual neural Dirichlet process mixture~\cite{lee2019neural} estimates the joint likelihood of a new sample and its label by which it decides whether an additional network component is required to handle the new sample.
There exist a few algorithms related to network augmentation, which increase the number of active channels adaptively with a sparse regularization to maintain the proper model size in a data-driven way~\cite{abati2020conditional, yan2021dynamically}.
Although these approaches are effective to deal with a long sequence of tasks, they require creating and storing additional network components and performing multiple forward pass computations for inference, which incur extra computational costs.

\subsection{Rehearsal and Pseudo-Rehearsal Methods}
\label{sec:rehearsal and pseudo-rehearsal}
Rehearsal methods save a small number of examples that represent the whole dataset effectively and utilize the stored data to learn a new task; the techniques in this category often focus on sample selection for rehearsal
Incremental Classifier and Representation Learning (iCaRL)~\cite{rebuffi2017icarl} selects the examples in a greedy manner with a cardinality constraint by approximating the mean of training data.
Also, \cite{aljundi2019gradient} identifies an exemplar set for the best approximation of the feasible region in the model parameter space induced by the whole data.  
On the other hand, pseudo-rehearsal techniques estimate the data distribution of the previous tasks and train new models using generated samples together with the real data of the new task to mitigate the class imbalance problem. 
For example, \cite{ostapenko2019learning} synthesizes the fake samples of the old tasks to balance the number of examples between old and new classes by employing class-conditional image synthesis networks~\cite{odena2017conditional}.

\subsection{Parameter Regularization Approaches}
\label{sec:parameter regularization}
Parameter regularization approaches encourage or penalize the updates of individual parameters depending on their importance.
To be more specific, Elastic Weight Consolidation (EWC)~\cite{kirkpatrick2017overcoming} estimates the rigidity or flexibility of each parameter based on the Fisher information matrix and employs its relevance to the previous tasks for model update.
Synaptic Intelligence (SI)~\cite{zenke2017continual} also provides the weight of each parameter by estimating the path integral along the optimization trajectory for individual parameters.
On the other hand, \cite{aljundi2018memory} calculates the importance in an unsupervised manner by computing the gradient of the magnitude of the network output.
However, as discussed in \cite{hsu2018re, van2019three}, parameter regularization methods empirically present relatively lower accuracy than the techniques in other categories, and might serve as a poor proxy for keeping the network output~\cite{benjamin2019measuring}.

\subsection{Knowledge Distillation Approaches}
\label{sec:functional regularization}
Knowledge distillation is originally proposed to achieve better generalization performance by transferring knowledge from a pretrained model (\ie, teacher) to a target network (\ie, student) through matching their logits~\cite{ba2014deep}, output distributions~\cite{hinton2015distilling}, intermediate activations~\cite{romero2014fitnets}, or attention maps~\cite{zagoruyko2016paying}.
In the class incremental learning scenarios, the model trained for old tasks is often considered as the teacher network while the student network corresponds to the model to be learned with the additional data in a new task from the initialization point given by the pretrained weights of the teacher.

As the teacher's knowledge, iCaRL~\cite{rebuffi2017icarl} employs the sigmoid output of each class in the previous tasks while~\cite{castro2018end, wu2019large, zhao2020maintaining} use the normalized softmax outputs with a temperature scaling.
Unified Classifier Incrementally via Rebalancing (UCIR)~\cite{hou2019learning} tackles the forgetting problem by maximizing the cosine similarity between the features embedded by the teacher and the student.
On the other hand, Learning without Memorizing (LwM)~\cite{dhar2019learning} minimizes the difference of the gradients given by the highest score classes in the old tasks with respect to intermediate feature maps, which is motivated by \cite{selvaraju2017grad}. 
Pooled Outputs Distillation (PODNet)~\cite{douillard2020podnet} minimizes the difference of the pooled intermediate features in the height and width directions instead of performing element-wise comparisons; such a relaxed objective using knowledge distillation turns out to be effective for class incremental learning.

Different from the previous works including the closely related method~\cite{park2021class}, our approach estimates the importance of each feature map for knowledge distillation and provides the theoretical foundation for the adaptive weighting of feature maps.
Note that the proposed knowledge distillation algorithm with the adaptive feature map weighting minimizes the loss increases by model updates over the previous tasks.


\section{Proposed Algorithm}
\label{sec:framework}
This section describes the proposed approach based on the knowledge distillation, which alleviates the catastrophic forgetting issue for class incremental learning.

\subsection{Problem Formulation}
\label{sub:problem formulation}

Convolutional neural networks~\cite{he2016deep, huang2017densely} typically contain multiple basic building blocks, each of which is composed of convolutions, batch normalizations~\cite{ioffe2015batch}, and activation functions.
A convolutional neural network learner at an incremental stage $t$, denoted by $M^t(\cdot)$, involves a classifier, $g(\cdot)$, and $L$ building blocks, $\{ f_{\ell}(\cdot) ; \ell=1,2,\dots ,L \}$, which is expressed as
\begin{align}
\hat{y} = M^t (x) & = g(h)  \nonumber \\
& = g \left( f_L \circ \cdots \circ f_{\ell} \circ \cdots \circ f_1(x) \right).
\end{align}
Note that the classifier maps $h \in \mathbb{R}^{d}$, an embedding vector given by the last building block, to a vector $\hat{y} \in \mathbb{R}^{n^t}$, where $n^t$ denotes the number of the observed classes until the $t^{\text{th}}$ task.
The $\ell^{\text{th}}$ building block $f_{\ell}(\cdot)$ also maps an input $Z_{\ell - 1}$ to a set of feature maps $Z_{\ell} = \{ Z_{\ell,1}, Z_{\ell,2}, \cdots, Z_{\ell,C_{\ell}} \}$.
As mentioned earlier, the network $M^t(\cdot)$ is initialized by the model parameters in the previous stage $M^{t-1}(\cdot)$ except the additional parameters introduced to handle new classes in the classifier.
For notational convenience, we introduce two fuctions $F_\ell(\cdot)$ and $G_\ell(\cdot)$ splitting the whole model into two subnetworks as follows:%
\begin{align}
M^t(x) =  G_\ell(Z_{\ell}) = G_\ell ( F_\ell (x) ).
\end{align}

At the time of training $M^t$, it optimizes the following objective function to mitigate catastrophic forgetting issue while learning new classes:
\begin{equation}
\displaystyle{\min_{M^t}\;\,  \mathbb{E}_{(x,y) \sim \mathcal{P}^{t-1}_{\text{data}}} \Big[ \mathcal{L}(M^{t}(x), y)\Big]},
	\label{eq:expectation_over_loss}
\end{equation}
where $\mathcal{P}^{t-1}_{\text{data}}$ is the data distribution of the past tasks until the $t-1^\text{st}$ incremental learning stage, and $\mathcal{L}(\cdot, \cdot)$ is a task-specific loss, \eg, classification loss in this work. 
However, it is well known that simply optimizing the objective while learning for new classes leads to catastrophic forgetting.
Although the discrepancy between two adjacent continual learning models, $M^{t-1}(\cdot)$ and $M^t(\cdot)$, may be small, the amount of forgetting gets more significant as the errors for the old tasks increase with respect to new models.

To analyze the effect of the distribution change from $Z_{\ell}$ to $Z'_{\ell}$ obtained respectively by $M^{t-1}(\cdot)$ and $M^t(\cdot)$ in the $\ell^\text{th}$ layer, 
we compute the loss given by $Z'_{\ell}$ via the first order Taylor approximation as follows:
\begin{align}
 	\mathcal{L} &(  G_\ell(Z'_{\ell}), y)  \approx \label{eq:Taylor_loss_approx} \\ 
	&  \mathcal{L}(G_\ell(Z_\ell), y) +  \sum_{c=1}^{C_\ell} \Big\langle \nabla_{Z_{\ell,c}} \mathcal{L}(G_\ell(Z_\ell), y), Z'_{\ell,c} - Z_{\ell,c} \Big\rangle _F,  \nonumber 
\end{align}
where $\langle \cdot , \cdot \rangle_F$ is the Frobenius inner product, \ie, $\langle A, B\rangle_F =  \text{tr}(A^\top B)$.
Note that  
$\nabla_{Z_{\ell,c}} \mathcal{L}(G_\ell(Z_\ell), y) \in R^{H_{\ell} \times W_{\ell}}$ is given by the backpropagation in $M^{t-1}(\cdot)$. 
Using Eq.~\eqref{eq:Taylor_loss_approx}, we define the loss increases due to the distribution shift in the $c^{\text{th}}$ channel of the $\ell^\text{th}$ layer, $\triangle \mathcal{L}(Z'_{\ell,c})$, as
\begin{align}
\label{eq:loss_approx} 
\triangle \mathcal{L}(Z'_{\ell,c}) & := \Big\langle \nabla_{Z_{\ell,c}} \mathcal{L} (G_\ell(Z_\ell), y), Z'_{\ell,c} - Z_{\ell,c} \Big\rangle _F.
\end{align}

Instead of solving the optimization problem in Eq.~\eqref{eq:expectation_over_loss}, we minimize the expected loss increases of old data incurred by the model updates at the $t^\text{th}$ incremental learning stage, which is given by
\begin{equation}
\displaystyle{\min_{M^t}\;\, \sum_{\ell=1}^{L}  \sum_{c=1}^{C_\ell}  \mathbb{E}_{(x,y) \sim \mathcal{P}^{t-1}_{\text{data}}} \Big[ \triangle \mathcal{L} (Z'_{\ell,c})} \Big]^2.
	\label{eq:expectation_over_approx_loss}
\end{equation}
Minimizing Eq.~\eqref{eq:expectation_over_approx_loss} via the standard stochastic gradient descent method requires the additional backpropagation process in $M^{t-1}(\cdot)$ unless we store the gradient of the feature maps with respect to $M^{t-1}(\cdot)$ for each example, which increases computational cost or memory overhead substantially.
We will discuss how to address this issue next.

\subsection{Practical Objective Function}
\label{sub:discrepancy}   
According to Eq.~\eqref{eq:expectation_over_approx_loss}, the distribution shift in each channel affects the loss function with a different magnitude. 
To avoid storing a large number of feature maps or performing additional backpropagation procedures, we derive an upper bound of the objective function in Eq.~\eqref{eq:expectation_over_approx_loss} as
\begin{align}
 &\mathbb{E}  \Big[\triangle \mathcal{L}(Z'_{\ell,c})\Big]  \nonumber \\
	 &= \mathbb{E}\Big[  \langle \nabla_{Z_{\ell,c}} \mathcal{L}(G_\ell(Z_\ell), y), Z'_{\ell,c} - Z_{\ell,c} \rangle _F  \Big]  
	\label{eq:expectation_over_loss_prime_upper_bound_0} \\ 
	&\leq \mathbb{E} \Big[ \| \nabla_{Z_{\ell,c}} \mathcal{L}(G_\ell(Z_\ell), y) \|_F \cdot \| Z'_{\ell,c} - Z_{\ell,c}\| _F \Big] 
	\label{eq:expectation_over_loss_prime_upper_bound_1} \\ 
	&\leq \sqrt{\mathbb{E}\Big[ \| \nabla_{Z_{\ell,c}} \mathcal{L}(G_\ell(Z_\ell), y) \|_F ^2 \Big] \cdot \mathbb{E}\Big[\| Z'_{\ell,c} - Z_{\ell,c}\| _F ^2\Big]}, 
	\label{eq:expectation_over_loss_prime_upper_bound_2} 
\end{align}
where the expectations are taken over $\mathcal{P}^{t-1}_{\text{data}}$.
Refer to the supplementary file for the detailed derivation.

Then, we define a more practical objective function, the upper bound of Eq.~\eqref{eq:expectation_over_approx_loss}, as follows:
\begin{align}
\displaystyle{\min_{M^t}}\; \sum_{\ell=1}^{L} \sum_{c=1}^{C_\ell} \,& I^t_{\ell,c} \mathbb{E}_{x \sim \mathcal{P}^{t-1}_{\text{data}}}\Big[\| Z'_{\ell,c} - Z_{\ell,c}\| _F ^2\Big],\label{eq:expectation_over_loss_prime_upper_bound_3} 
\end{align}
where the importance is given by 
\begin{equation}
I^t_{\ell,c} := \mathbb{E}_{(x,y)\sim \mathcal{P}^{t-1}_{\text{data}}} \Big[ \| \nabla_{Z_{\ell,c}} \mathcal{L}(G_\ell(Z_\ell), y) \|_F ^2 \Big], \nonumber
\end{equation}
which is estimated by a Monte Carlo integration at the previous stage using $M^{t-1}(\cdot)$ and employed for training $M^t$.
Note that we need marginal cost for storing $I^t_{\ell,c}$ since it is just a scalar value corresponding to each channel.
\begin{figure}[t!]
\centering
\includegraphics[width=0.95\linewidth]{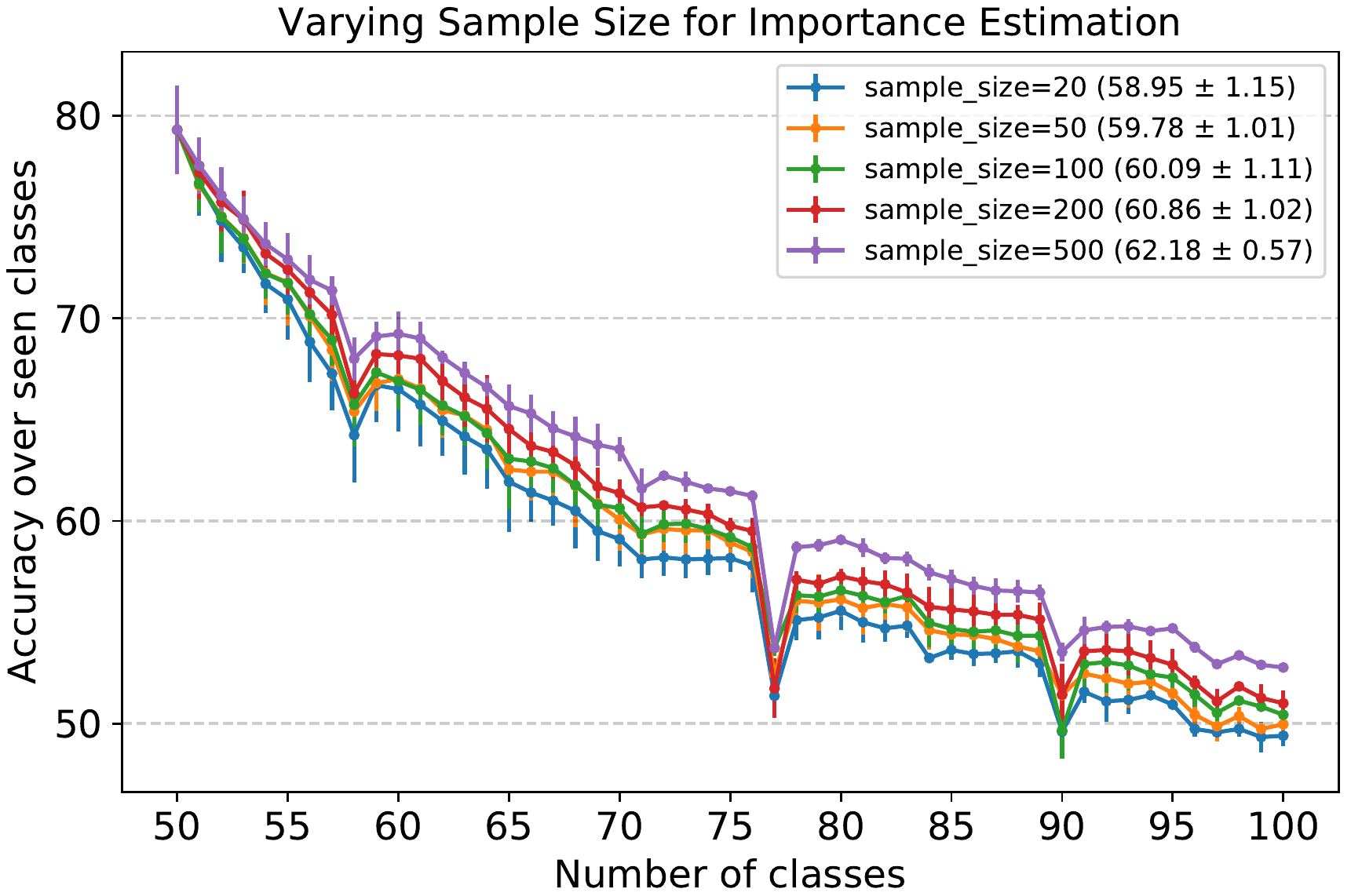}
\caption{Incremental accuracy on CIFAR100, 50 incremental stages, with ResNet-32 over three different class orders by varying different sample sizes for importance estimation. The legend indicates the mean and standard deviation of the results.}
\label{fig:gradient_sample}
\end{figure}
The importance $I^t_{\ell,c}$ can be interpreted as a weight factor because the larger value of $I^t_{\ell,c}$ incurs the bigger changes in the loss given by the same perturbation in the feature map.  
Furthermore, Figure~\ref{fig:gradient_sample} shows that a larger sample sizes estimates more reliable importance and consequently leads to better generalization performance. 

\subsection{More Robust Objective Function}
\label{sub:more_practical}   
Although the optimization of the objective function in Eq.~\eqref{eq:expectation_over_loss_prime_upper_bound_3} is effective to maintain the representations of the examples in the previous tasks, we usually have a limited number of exemplars and face critical challenge in reducing catastrophic forgetting.
Therefore, we also adopt the data in the current task together with the ones in the previous tasks, which leads to the following optimization problem:
\begin{align}
\displaystyle{\min_{M^t}}\; \sum_{\ell=1}^{L} \sum_{c=1}^{C_\ell} \,& I^t_{\ell,c}
 \mathbb{E}_{x\sim \mathcal{P}^{t}_{\text{data}}} \Big[\| Z'_{\ell,c} - Z_{\ell,c}\| _F ^2 \Big],
 \label{eq:expectation_over_loss_prime_upper_bound_4}
\end{align}
where its only difference from Eq.~\eqref{eq:expectation_over_loss_prime_upper_bound_3} is the distribution for the computation of the expectation.
Note that the use of the examples in the current task for the purpose of preserving the representations in the previous stages is also motivated by other knowledge distillation approaches~\cite{li2017learning, dhar2019learning, hou2019learning, douillard2020podnet}.
This strategy is also supported by the following proposition.

\begin{proposition}
\vspace{0.2cm}
Let $\mathcal{Y}^t_{\text{old}}$, $\mathcal{Y}^t_{\text{new}}$, and $q_{\text{data}}^t$  be a set of previously observed classes ($\mathcal{Y}^{t-1}$), the set of newly observed classes at $t^\text{th}$ incremental learning stage, and the distribution of the data in the current task,  respectively.
Let the distribution of labels in $\mathcal{Y}^t$ be a mixture of two prior distributions corresponding to $\mathcal{Y}_{\text{old}}^t$ and $\mathcal{Y}_{\text{new}}^t$ with weights $\phi_{\text{old}}^t$ and $\phi_{\text{new}}^t$, respectively.
Assuming that the data generating process for each class, $\mathcal{P}_\text{data}^t(x|y)$, does not change over the incremental stages,
the objective in Eq.~\eqref{eq:expectation_over_loss_prime_upper_bound_3} with the weight factor $\phi_\text{old}^t$ is bounded above by that in  Eq.~\eqref{eq:expectation_over_loss_prime_upper_bound_4}.
\label{propostion:1}
\end{proposition}
\begin{proof}
By the definition of the label distribution for $\mathcal{Y}^t$, the class prior distribution $\mathcal{P}_{\text{data}}^t(y)$ is given by 
\begin{align}
\mathcal{P}_{\text{data}}^t (y) &= \mathcal{P}_{\text{data}}^t (y| \mathcal{Y}_\text{old}^t )  \phi_\text{old}^t +\mathcal{P}_{\text{data}}^t (y|\mathcal{Y}_\text{new}^t ) \phi_\text{new}^t \nonumber \\
&= \mathcal{P}_{\text{data}}^{t-1} (y) \phi_\text{old}^t + q_\text{data}^t(y)  \phi_\text{new}^t. 
\label{eq:proposition_0}
\end{align}
Based on this equation, $\mathcal{P}_\text{data}^t(x)$ can be expressed as
\begin{align}
&\mathcal{P}_{\text{data}}^t (x)  \nonumber \\
&= \sum_{y \in \mathcal{Y}_\text{old}^t} \mathcal{P}_{\text{data}}^t (y) \mathcal{P}_{\text{data}}^t (x|y) + \sum_{y \in \mathcal{Y}_\text{new}^t} \mathcal{P}_{\text{data}}^t (y) \mathcal{P}_{\text{data}}^t (x|y) \nonumber \\
&=  \phi_\text{old}^t \cdot \sum_{y \in \mathcal{Y}^{t-1}} \mathcal{P}_{\text{data}}^{t-1} (y) \mathcal{P}_{\text{data}}^{t-1} (x|y) + \nonumber \\ 
 & \quad \, \phi_\text{new}^t \cdot \sum_{y \in \mathcal{Y}^t_{\text{new}}} q_\text{data}^t(y) q_\text{data}^t(x|y) \nonumber \\
&=   \phi_\text{old}^t \cdot \mathcal{P}_{\text{data}}^{t-1} (x) + \phi_\text{new}^t \cdot q_\text{data}^t(x). \label{eq:proposition_1} 
\end{align}
Then, by denoting $\Delta Z_{\ell, c} := \|Z'_{\ell,c} - Z_{\ell,c}\|_F ^2$, which is always non-negative, we obtain the following inequality:
\begin{align}
&\mathbb{E}_{x \sim \mathcal{P}_{\text{data}}^{t-1}}\Big[ \Delta Z_{\ell,c} \Big] \nonumber \\
&\leq
\mathbb{E}_{x \sim \mathcal{P}_{\text{data}}^{t-1}}\Big[ \Delta Z_{\ell,c} \Big] 
+  
\frac{\phi_\text{new}^t}{\phi_\text{old}^t}  \mathbb{E}_{x \sim q_{\text{data}}^{t}}\Big[ \Delta Z_{\ell,c} \Big]. \label{eq:proposition_2}
\end{align}
By multiplying both sides of Eq.~\eqref{eq:proposition_2} by $\phi_\text{old}^t$ and using Eq.~\eqref{eq:proposition_1}, we have   
\begin{align}
&\phi_\text{old}^t\mathbb{E}_{x \sim \mathcal{P}_{\text{data}}^{t-1}}\Big[ \Delta Z_{\ell,c} \Big]  \nonumber \\ 
&\leq
\phi_\text{old}^t\mathbb{E}_{x \sim \mathcal{P}_{\text{data}}^{t-1}}\Big[ \Delta Z_{\ell,c} \Big] 
+  
\phi_\text{new}^t \mathbb{E}_{x \sim q_{\text{data}}^{t}}\Big[ \Delta Z_{\ell,c} \Big] \nonumber \\
&=\mathbb{E}_{x \sim \mathcal{P}_{\text{data}}^{t}}\Big[ \Delta Z_{\ell,c} \Big],  \label{eq:proposition_3}
\end{align}
which concludes the proof. 
\end{proof}
Proposition~\ref{propostion:1} implies that the optimization procedure of Eq.~\eqref{eq:expectation_over_loss_prime_upper_bound_4} minimizes the upper bound of the loss increases over the previous tasks using the examples in the previous and current tasks effectively.
\begin{table*}[ht!]
\caption{Performance comparison between the proposed algorithm, denoted by AFC, and other state-of-the-art methods on CIFAR100. We run experiments using  three different class orders and report their averages and standard deviations. Methods with an asterisk * report the results from our reproductions with their official codes. Red and black bold-faced numbers represent the best and second-best performance in each column.}
\label{tab:Cifar}
\centering
\scalebox{0.90}{
\begin{tabular}{@{}l|cccc@{}}
 \toprule
 & \multicolumn{4}{c}{CIFAR100}\\
  & 50 stages & 25 stages & 10 stages & 5 stages\\
 \multicolumn{1}{r|}{New classes per stage} & 1 & 2 & 5 & 10\\
 \midrule
iCaRL~\cite{rebuffi2017icarl} & 44.20\mypm{}0.98    & 50.60\mypm{}1.06  & 53.78\mypm{}1.16  & 58.08\mypm{}0.59\\ 
 BiC~\cite{wu2019large} & 47.09\mypm{}1.48 & 48.96\mypm{}1.03 & 53.21\mypm{}1.01  & 56.86\mypm{}0.46\\
 UCIR\,{\scriptsize (NME)}~\cite{hou2019learning}  & 48.57\mypm{}0.37 & 56.82\mypm{}0.19 & 60.83\mypm{}0.70 & 63.63\mypm{}0.87\\
 UCIR\,{\scriptsize (CNN)}~\cite{hou2019learning}& 49.30\mypm{}0.32 & 57.57\mypm{}0.23 & 61.22\mypm{}0.69 & 64.01\mypm{}0.91\\
 Mnemonics~\cite{liu2020mnemonics} & - & 60.96\mypm{}0.72 & 62.28\mypm{}0.43 & 63.34\mypm{}0.62 \\
 GDumb*~\cite{prabhu2020gdumb} & 59.76\mypm{}1.49 & 59.97\mypm{}1.51 & 60.24\mypm{}1.42 & 60.70\mypm{}1.53 \\
 PODNet\,{\scriptsize (NME)}~\cite{douillard2020podnet} & 61.40\mypm{}0.68 & 62.71\mypm{}1.26 & 64.03\mypm{}1.30 & 64.48\mypm{}1.32\\
 PODNet\,{\scriptsize (CNN)}~\cite{douillard2020podnet} &  57.98\mypm{}0.46 & 60.72\mypm{}1.36 & 63.19\mypm{}1.16 & 64.83\mypm{}0.98\\
  PODNet\,{\scriptsize (NME)}*~\cite{douillard2020podnet} & 56.78\mypm{}0.41 & 59.54\mypm{}0.66 & 63.27\mypm{}0.69 & 65.32\mypm{}0.65\\
 PODNet\,{\scriptsize (CNN)}*~\cite{douillard2020podnet} & 57.86\mypm{}0.38 & 60.51\mypm{}0.62 & 62.78\mypm{}0.78 & 64.62\mypm{}0.65\\
  GeoDL*~\cite{simon2021learning} & 52.28\mypm{}3.91 & 60.21\mypm{}0.46 & 63.61\mypm{}0.81 & 65.34\mypm{}1.05 \\ 
 SSIL*~\cite{ahn2021ss} & 53.64\mypm{}0.77 & 58.02\mypm{}0.79 & 61.52\mypm{}0.44 & 63.02\mypm{}0.59 \\
 \hdashline
 AFC\,{\scriptsize (NME)}& \red{62.58\mypm{}1.02} & \red{64.06\mypm{}0.73} & \black{64.29\mypm{}0.92 }& \black{65.82\mypm{}0.88}\\
 AFC\,{\scriptsize (CNN)}& \black{62.18\mypm{}0.57} & \black{63.89\mypm{}0.93} & \red{64.98\mypm{}0.87} & \red{66.49\mypm{}0.81}\\
 \bottomrule
\end{tabular}
}
\end{table*}

\subsection{Final Loss Function}
The final goal of $M^t(\cdot)$ is to minimize the combination of the classification loss and the discrepancy loss as
\begin{equation}
	\mathcal{L}^t_{\text{total}} = \mathcal{L}^t_{\text{cls}} + \lambda_\text{disc}  \cdot \lambda^t  \mathcal{L}^{t}_\text{disc},
	\label{eq:final_loss}
\end{equation}
where $\lambda_\text{disc}$ is a hyperparameter and an adaptive weight, $\lambda^t$, is set to $\sqrt{\frac{n^t}{n^t-n^{t-1}}}$ as proposed in~\cite{hou2019learning, douillard2020podnet}. 
For the classification loss $\mathcal{L}_\text{cls}^t$, we adopt a local similarity classifier (LSC)~\cite{douillard2020podnet} by adaptively aggregating the cosine similarities obtained from multiple class embedding vectors, which is given by
\begin{equation}
 \hat{y}_k = \sum_{j}s_{k,j}\,\langle \theta_{k,j},h \rangle\,,
\quad
s_{k,j} =\frac{\exp\,\langle \theta_{k,j},h \rangle}{\sum_{i} \exp\,\langle \theta_{k,i},h \rangle}\,, 
	\label{eq:podnet_local_similarity}
\end{equation}
where $\theta_{k,j}$ is the $j^{\text{th}}$ normalized class embedding vector for the $k^{\text{th}}$ class and $\hat{y}_k$ is the score for the $k^{\text{th}}$ class. 
Then, the classification loss $\mathcal{L}_\text{cls}^t$ is defined with a label $g^{(b)}$ as
\begin{align}
	\mathcal{L}^t_{\text{cls}} &= \frac{1}{B} \sum_{b=1}^B \bigg[ -\log \frac{\exp \,(\eta(\hat{y}_{g^{(b)}}^{(b)} - \delta))}{\sum_{i \neq g^{(b)}} \exp\, (\eta \hat{y}_i^{(b)}) }\bigg]_{+},
	\label{eq:LSC_loss}  
\end{align}
where $\eta$ is a learnable scaling parameter, $\delta$ is a constant to encourage a larger inter-class separation, $B$ is the mini-batch size, and $[\cdot]_+$ denotes the ReLU activation function.
Also, the discrepancy loss, $\mathcal{L}_\text{disc}^t$, is given by
\begin{align}
\mathcal{L}^{t}_\text{disc}& : =  \frac{1}{B}\sum_{b=1}^B \sum_{\ell=1}^L  \sum_{c=1}^{C_\ell} \tilde{I}^t_{\ell,c}\| Z^{'(b)}_{\ell,c} - Z_{\ell,c}^{(b)} \| _F ^2, \label{eq:old_new_example}
\end{align}
where $\tilde{I}^t_{\ell,c}$ is introduced to balance the scale of the importances across layers, which is given by
\begin{equation}
	\tilde{I}^t_{\ell,c} = \frac{I^t_{\ell,c}}
	{\frac{1}{C_\ell}\sum_{c=1}^{C_{\ell}} I^t_{\ell,c}}.
	\label{eq:normalize_importance}
\end{equation}
Algorithm~\ref{alg:importance_estimation} presents the details of the importance estimation procedure after training $M^t(\cdot)$.

\begin{algorithm}[t]
  \caption{Normalized Importance Estimation}
  \label{alg:importance_estimation}
  \begin{algorithmic}[1]
  \State{\bf Input:} current model $M^{t}(\cdot)$, current dataset $D^{t}$, exemplar sets for the previous tasks $E^{t}$  
        \State {\bf Initialization:} $I^t_{\ell,c} \gets 0 $ and $D \gets D^{t} \cup E^{t}$    
         \For{all $(x, y) \in D$}
        \State Compute $\mathcal{L}^t_{\text{cls}}$ by Eq.~\eqref{eq:LSC_loss} 
        \State Perform backward through $M^t(\cdot)$
        \For{$\ell \in 1,2, \ldots, L$} 
         \State Update the importance: 
          \State $I^t_{\ell,c} \gets I^t_{\ell,c}  + \| \nabla_{Z_{\ell,c}} \mathcal{L}^t_{\text{cls}} \|^2_F$
        \EndFor 	
         \EndFor
         \For{$\ell \in 1,2, \ldots, L$} 
          \State Normalize the importance $I^t_{\ell,c}$ by Eq.~\eqref{eq:normalize_importance}
           \EndFor
         \State{\bf Output:} Normalized importance $\{ \tilde{I}^t_{\ell,c}\}$ employed for training $M^{t+1}(\cdot)$
  \end{algorithmic}
\end{algorithm}

\subsection{Management of Exemplar Set}
To maintain the exemplar set for each class after training $M^t(\cdot)$, we employ the nearest-mean-of-examplers selection rule~\cite{rebuffi2017icarl} as in other class incremental learning methods. 
For inference, we employ two methods; one is the nearest-mean-of-exemplars classification rule~\cite{rebuffi2017icarl} and the other is based on the classifier probabilities as discussed in \cite{hou2019learning,douillard2020podnet}, which we call NME and CNN, respectively.

\subsection{Discussion about Computational Cost}
As in most approaches based on knowledge distillation, we perform 2 forward passes and 1 backward pass during training.
Different from the previous approaches, at the end of the training phase of each task, we update the importance by performing a pair of forward and backward pass computations using the old exemplars and examples in the current task as presented in Algorithm~\ref{alg:importance_estimation}.
Therefore, the additional computation of our method is only for the importance update, which happens only once per task; it is negligible compared to the training cost of each task.

\section{Experiments}
\label{sec:experiments}
This section compares our algorithm with the previous approaches on the standard datasets, and analyzes the experimental results thorough ablation studies.

\subsection{Dataset}
\label{sec:dataset} 
We employ the standard datasets, CIFAR100~\cite{cifar09} and ImageNet100/1000~\cite{ILSVRC15}, for evaluation of class incremental learning algorithms.
The CIFAR100 dataset contains 50,000 and 10,000 32 $\times$ 32 color images for training and testing.
The ImageNet1000 dataset consists of 1,281,167 images for training and 50,000 images for validation across 1,000 classes.
ImageNet100 is a subset of ImageNet1000 with randomly selected 100 classes.
For image preprocessing and class order, we follow the protocol in PODNet~\cite{douillard2020podnet} for fair comparisons.
Except the initial task containing half of the total classes for all datasets, the remaining classes are evenly distributed across tasks.
After each incremental stage, we store 20 images per class and construct exemplar sets unless stated otherwise.

\subsection{Implementation Details}
\label{sec:details}
The proposed method is implemented based on the publicly available official code of PODNet and tested using a NVIDIA Tesla V100 GPU, without modifying any experimental settings and network structures for fair comparisons.
For all datasets, we adopt the SGD optimizer with a momentum of 0.9, an initial learning rate of 0.1 with a cosine annealing schedule, and  a batch size of 128.
We train a 32-layer ResNet~\cite{he2016deep} for 160 epochs with a weight decay of 0.0005 for CIFAR100, and an 18-layer ResNet~\cite{he2016deep} for 90 epochs with a weight decay of 0.0001 for ImageNet datasets, where all compared methods also use the same backbones. 
Following the heuristics proposed in~\cite{zagoruyko2016paying, douillard2020podnet}, we normalize each feature map using its Frobenius norm to compute the discrepancy loss to further improves performance.
We select the same layers with PODNet for knowledge distillation.
In case of $\lambda_\text{disc}$,  we set it to $4.0$ for CIFAR100 and $10.0$ for ImageNet.

For comparisons, we report \emph{average incremental accuracy}~\cite{rebuffi2017icarl}---the average of the accuracies over seen classes across all incremental stages, where the results of other methods except GeoDL~\cite{simon2021learning}, DDE~\cite{hu2021distilling}, and SSIL~\cite{ahn2021ss} are retrieved from~\cite{douillard2020podnet}  if not specified otherwise.
Our algorithm is referred to as Adaptive Feature Consolidation (AFC) to present experimental results in this section.

\begin{table*}[!t]
\caption{Performance comparison between AFC and other state-of-the-art algorithms on ImageNet100 and ImageNet1000.}
\label{tab:Imagenet}
\centering
\scalebox{0.90}{
\begin{tabular}{@{}l|cccc|cc@{}}
 \toprule
 & \multicolumn{4}{c|}{ImageNet100} & \multicolumn{2}{c}{ImageNet1000} \\
  & 50 stages & 25 stages & 10 stages & 5 stages & 10 stages & 5 stages\\
 \multicolumn{1}{c|}{New classes per stage} & 1 & 2 & 5 & 10 & 50 & 100\\
 \midrule
 iCaRL~\cite{rebuffi2017icarl}         & 54.97 & 54.56 & 60.90  & 65.56 & 46.72   & 51.36 \\
 BiC~\cite{wu2019large} & 46.49 & 59.65 & 65.14  & 68.97 & 44.31   & 45.72\\
 UCIR\,{\scriptsize (NME)}~\cite{hou2019learning}    & 55.44 & 60.81 & 65.83  & 69.07 & 59.92  & 61.56 \\
 UCIR\,{\scriptsize (CNN)}~\cite{hou2019learning}     & 57.25 & 62.94 & 70.71  & 71.04 & 61.28   & 64.34 \\
 UCIR\,{\scriptsize (CNN)} + DDE~\cite{hu2021distilling} & - & - & 70.20 & 72.34 & \black{65.77} & \black{67.51} \\
 Mnemonics~\cite{liu2020mnemonics} & - & 69.74 & 71.37 & 72.58 & 63.01 & 64.54 \\
 PODNet\,{\scriptsize (CNN)}~\cite{douillard2020podnet}                & \black{62.48} & 68.31 & 74.33 & 75.54 & 64.13 & 66.95 \\
PODNet\,{\scriptsize (CNN)}~\cite{douillard2020podnet} + DDE~\cite{hu2021distilling} & - & - & \black{75.41} & \black{76.71} & 64.71 &66.42 \\ 
 GeoDL~\cite{simon2021learning} & - & \black{71.72} & 73.55 & 73.87 & 64.46 &65.23 \\
\hdashline 
AFC\,{\scriptsize (CNN)}                & \red{72.08} & \red{73.34} & \red{75.75} & \red{76.87} & \red{67.02} & \red{68.90}\\
 \bottomrule
\end{tabular}
}
\end{table*}
%
\subsection{Results on CIFAR100}
\label{sec:main_cifar100}
We test the proposed method, AFC, in various class incremental learning scenarios.
Table~\ref{tab:Cifar} clearly shows that the proposed approach achieves the highest accuracy in all the tested numbers of incremental stages.
Note that the reproduced results of PODNet with NME inference are worse than the reported ones, even though we use their official code.
Compared with the reproduced PODNet and the other techniques, the performance gap becomes larger as the number of stages grows, which implies that the proposed method is more robust to challenging scenarios than the previous methods. 
AFCs with NME and CNN for inference outperform the reported results of PODNet in \cite{douillard2020podnet} and achieve the state-of-the-art performance.
\begin{figure}[t!]
\centering
\includegraphics[width=0.8\linewidth]{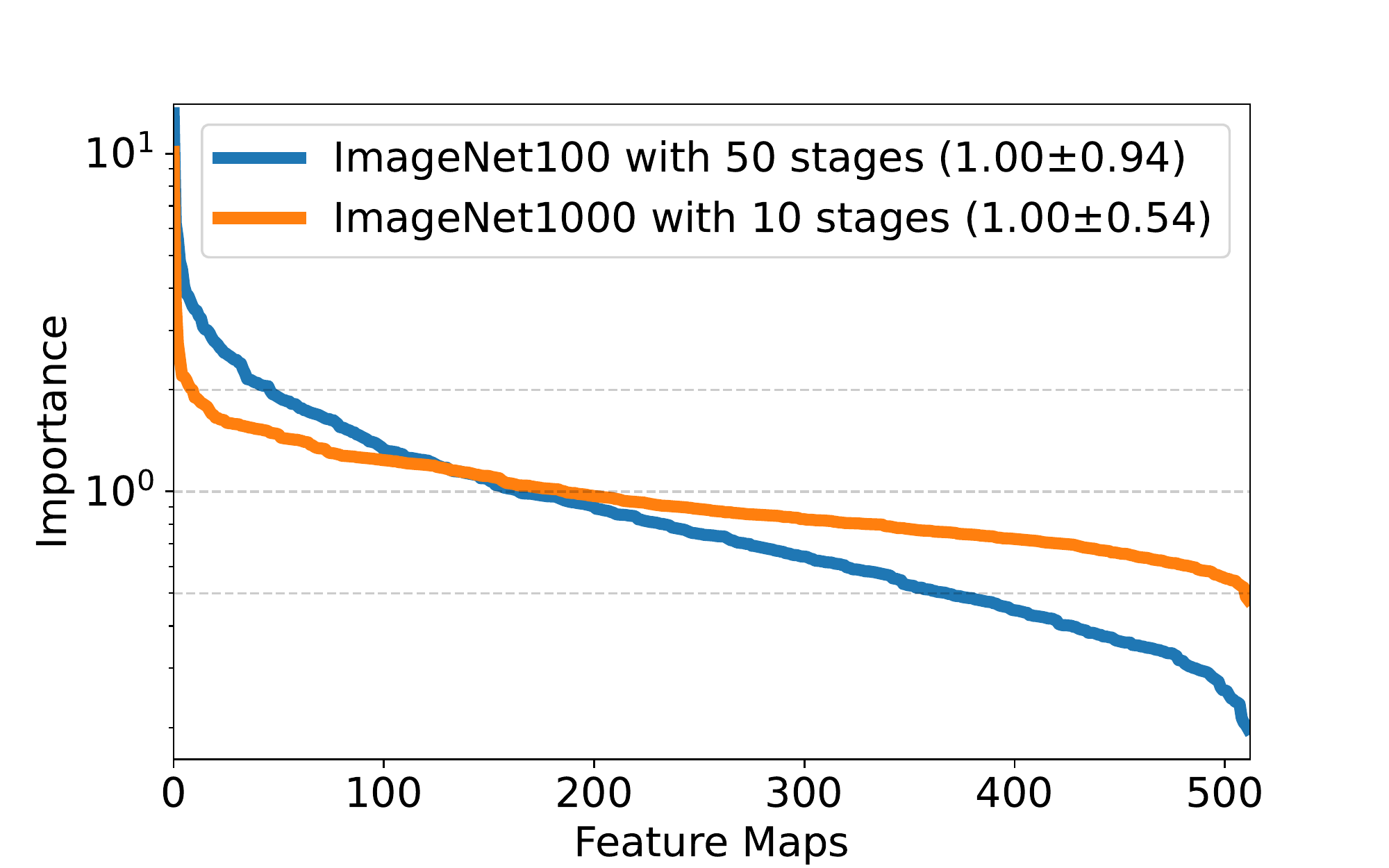}
\caption{The sorted importance of the feature maps in the last block on ImageNet100 with 50 stages, and ImageNet1000 with 10 stages.}
\label{fig:imagenet_importance}
\end{figure}
\begin{table}[!t]
\caption{Sensitivity analysis regarding the memory budget per class on CIFAR100 with 50 stages.
  AFC with Eq.~\eqref{eq:expectation_over_loss_prime_upper_bound_3} optimizes the discrepancy loss over the only exemplar sets without using the data in the current task.}
  \label{tab:ablation_memorysize}
 
\centering
  \scalebox{0.90}{
 \begin{tabular}{@{}lccccc@{}}
 \toprule
Memory budget per class & 5     & 10    & \textbf{20}    & 50    \\
 \midrule
iCaRL \cite{rebuffi2017icarl}      & 16.44 & 28.57 & 44.20 & 48.29  \\
BiC~\cite{wu2019large}       & 20.84  & 21.97  & 47.09  & 55.01  \\
UCIR\,{\scriptsize (CNN)}~\cite{hou2019learning} & 22.17 & 42.70 & 49.30 & 57.02 \\
PODNet\,{\scriptsize (CNN)}~\cite{douillard2020podnet}  & 35.59 & 48.54 & 57.98 & 63.69  \\
\hdashline
AFC\,{\scriptsize (CNN)} with \eqref{eq:expectation_over_loss_prime_upper_bound_3}   & \black{40.37} & \black{55.24} & \black{61.73} & \red{65.10} \\
AFC\,{\scriptsize (CNN)}    & \red{44.66} & \red{55.78} & \red{62.18} & \black{65.07} \\
\bottomrule
\end{tabular}
}
\end{table}
%
\subsection{Results on ImageNet}
\label{sec:main_imagenet}
Table~\ref{tab:Imagenet} presents the results from all algorithms including AFC on ImageNet100 and ImageNet1000 our algorithm; AFC shows outstanding performance compared to all the compared methods.
Moreover, AFC exceeds the previous methods by large margins for 50 stages on ImageNet100, which is more prone to suffer from catastrophic forgetting.
These results imply that the proposed approach has great potential for large-scale problems.

\subsection{Analysis and Ablation Studies}
\label{sec:analysis}
We analyze the estimated importance and discuss the effects of memory size for exemplar sets, feature distillation methods, varying initial task sizes, and discrepancy loss. 

\vspace{-0.2cm}
\paragraph{Importance visualization}
Figure~\ref{fig:imagenet_importance} visualizes the importances of individual feature maps in the last building block, which are sorted after training the final task.
It demonstrates that the standard deviation on ImageNet1000 is smaller than that on ImageNet100.
This is because more features are needed to effectively learn the knowledge from the examples in ImageNet1000  since the dataset is more challenging. 
On the other hand, although it is not shown in the figure, we also observe the relatively higher standard deviation in the early stages for both datasets, which is partly because the features in the early stages are less critical to the loss.
\begin{table}[!t]
  \caption{Performance comparison of feature distillation approaches on CIFAR100 with 50 stages.} 
\label{tab:feature_map}
\centering
\scalebox{0.9}{
\begin{tabular}{@{}lcc@{}} 
 \toprule
 Inference methods    & NME & CNN\\
 \midrule
 Fine-tuning                  & 41.56  & 40.76\\
 Uniform Importance            & 42.21  & 40.81 \\
 GradCam Preserving~\cite{dhar2019learning}    & 41.89  & 42.07\\
 Gram Matrix Preserving~\cite{johnson2016perceptual}   & 41.74 & 40.80 \\
 Channel Preserving~\cite{douillard2020podnet}      & 55.91  & 50.34\\
 Gap Preserving~\cite{douillard2020podnet}              & 57.25  & 53.87\\
PODNet~\cite{douillard2020podnet}         & \black{61.42} & \black{57.64}\\
PODNet*~\cite{douillard2020podnet}         & 57.15 & 57.51\\
\hdashline
 AFC  & \red{62.58} & \red{62.18} \\
 \bottomrule
\end{tabular}
}
\end{table}
%
\begin{figure*}[!ht]
\centering
    \begin{subfigure}{0.45\linewidth}
    \includegraphics[width=1.0\linewidth]{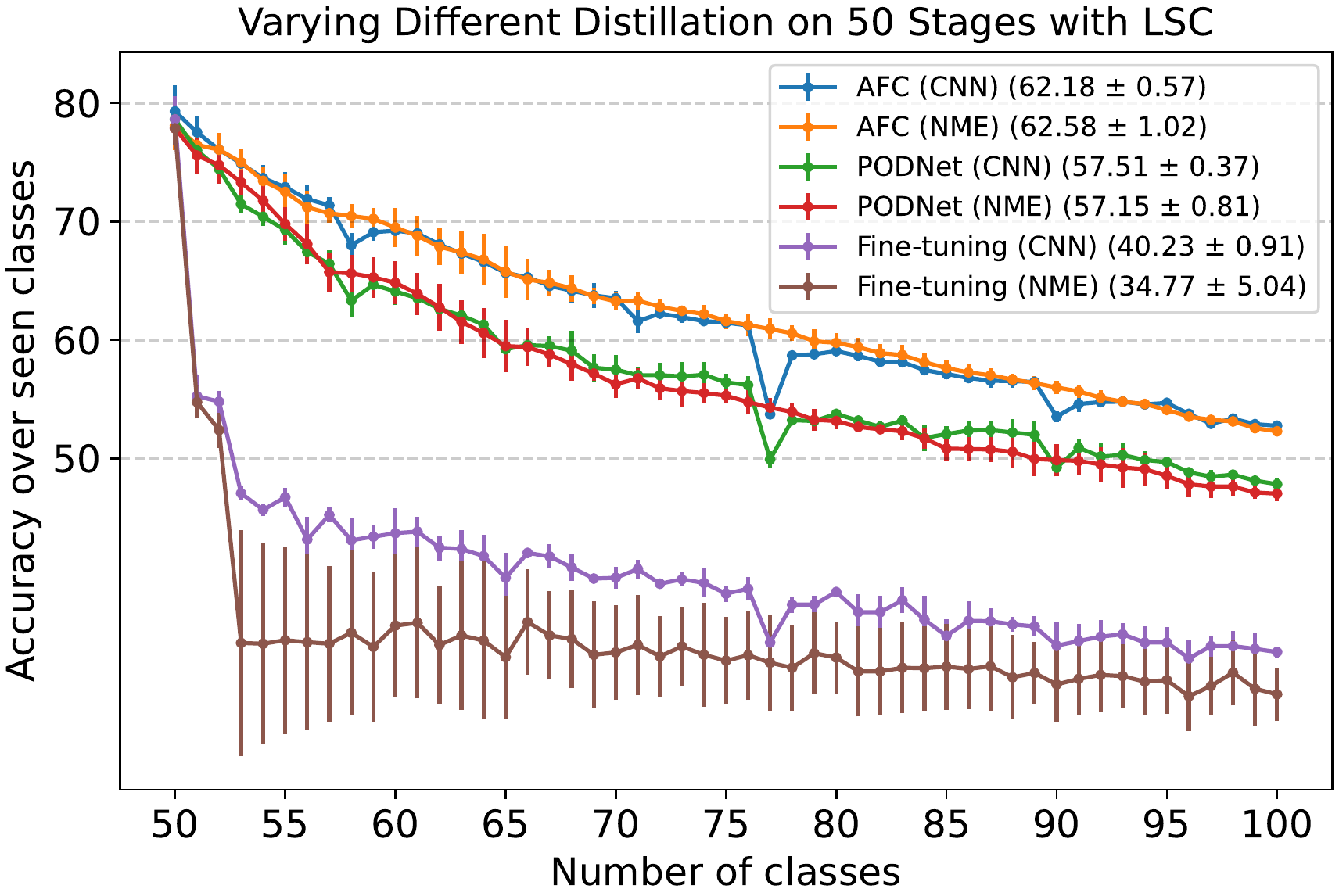} 
    \label{fig:3a}
    \end{subfigure}\hfill
    \begin{subfigure}{0.45\linewidth}
    \includegraphics[width=1.0 \linewidth]{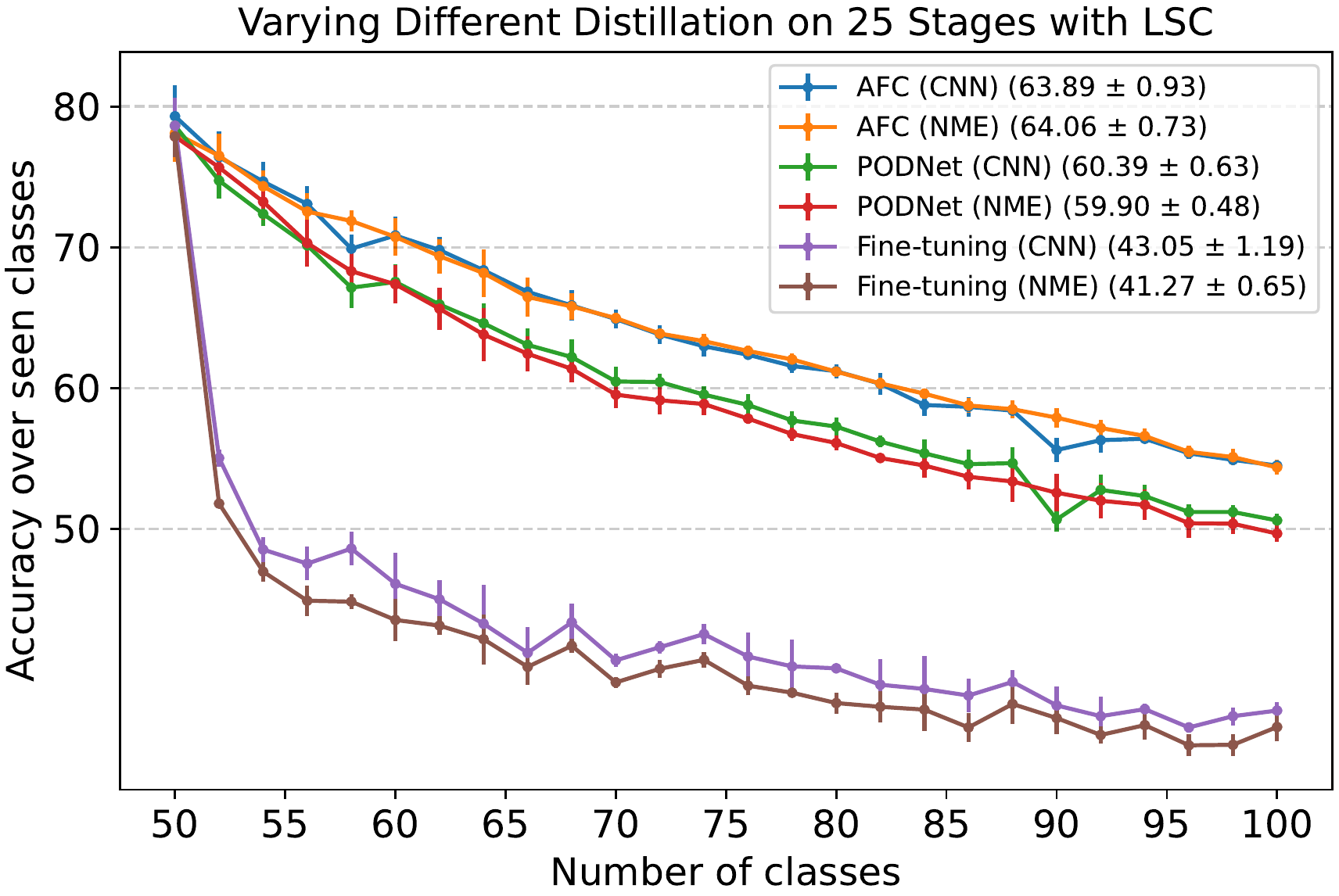}
    \label{fig:3b}
    \end{subfigure}\hfill
    \vspace{-0.2cm}
   \caption{Average accuracies of three different class orders during class incremental learning on CIFAR100 with 50 and 25 stages.}
\vspace{-0.2cm}
    \label{fig:plot_ablation}
\end{figure*} 
\begin{table}[!t]
\caption{Performance comparison between AFC and the state-of-the-art algorithms on CIFAR100 by varying the number of classes in the initial task while each of the remaining tasks only contains a single class.} 
\label{tab:varying_initial_task}
\centering
\scalebox{0.90}{
\setlength\tabcolsep{3.5pt} 
\begin{tabular}{@{}l|cccc@{}}
 \toprule
      & \multicolumn{4}{c}{CIFAR100} \\
 & 80 stages & 70 stages & 60 stages & 50 stages \\
\multicolumn{1}{c|}{Initial task size} & 20 & 30 & 40 & 50\\
 \midrule
 iCaRL \cite{rebuffi2017icarl}     & 41.28 & 43.38 & 44.35 & 44.20\\
BiC \cite{wu2019large}       & 40.95  & 42.27 & 45.18 & 47.09\\
UCIR\,{\scriptsize (NME)} \cite{hou2019learning} & 40.81 & 46.80 & 46.71 & 48.57\\
UCIR\,{\scriptsize (CNN)} \cite{hou2019learning} & 41.69 & 47.85 & 47.51 & 49.30\\
 GDumb*~\cite{prabhu2020gdumb} & \black{52.04} & 55.15 & 57.44 & 59.76\\
PODNet\,{\scriptsize (NME)}~\cite{douillard2020podnet} & 49.03 & 55.30 & 57.89 & 61.40\\
PODNet\,{\scriptsize (CNN)}~\cite{douillard2020podnet}  & 47.68 & 52.88 & 55.42 & 57.98\\
PODNet\,{\scriptsize (NME)}*~\cite{douillard2020podnet}  & 45.88 & 52.18 & 54.22 & 56.78 \\
PODNet\,{\scriptsize (CNN)}*~\cite{douillard2020podnet} & 48.37 & 53.28 & 55.86 & 57.86 \\
\hdashline
AFC\,{\scriptsize (NME)}& 51.31 & \black{57.05} & \black{60.06} & \red{62.58} \\
AFC\,{\scriptsize (CNN)}& \red{52.90} & \red{57.61} & \red{60.27} & \black{62.18} \\
\bottomrule
\end{tabular}
}
\end{table}

\vspace{-0.2cm}
\paragraph{Effect of memory budget}
To assess the effectiveness of the proposed method, we vary the memory budget. 
Table~\ref{tab:ablation_memorysize} illustrates that AFC outperforms the state-of-the-art methods by large margins under several different memory budgets.
AFC optimized using the exemplar sets together with the current data tends to achieve better generalization performance than AFC with Eq.~\eqref{eq:expectation_over_loss_prime_upper_bound_3} when the memory budget is small.
However, when the budget per class increases as large as $50$, AFC with Eq.~\eqref{eq:expectation_over_loss_prime_upper_bound_3} achieves slightly better accuracy, which is probably because Eq.~\eqref{eq:expectation_over_loss_prime_upper_bound_3} leads to a tighter upper bound and the use of the examples in the current task is not particularly helpful.

\vspace{-0.2cm}
\paragraph{Comparison with feature distillation approaches}
We present an experimental result by applying different feature distillation approaches on CIFAR100 with 50 stages.
Table~\ref{tab:feature_map} shows the performance of class incremental learning algorithms based on knowledge distillation, where AFC is compared with 7 different methods including ``Fine-tuning" and ``Uniform Importance" with all the importances set to 1.
According to the results, AFC outperforms all the compared algorithms.
In addition, Figure~\ref{fig:plot_ablation} also implies that AFC is more robust to catastrophic forgetting than PODNet consistently throughout the incremental stages. 

\vspace{-0.2cm}
\paragraph{Results by varying initial task sizes}
\label{appendix:effect_initial_task_size}
We test the proposed method by reducing the number of classes in the initial task while each of the remaining tasks consists of only one class, which makes the problem more challenging by increasing the total number of the incremental stages.  
Table~\ref{tab:varying_initial_task} presents that AFC outperforms the previous algorithms consistently.

\begin{table}[t]
\centering
\caption{Sensitivity analysis of $\lambda_\text{disc}$ on CIFAR100 with 50 stages.}
\label{tab:ablation}
\scalebox{0.90}{
\begin{tabular}{@{}c|ccccc@{}}
 \toprule
 \multicolumn{1}{c}{$\lambda_\text{disc}$} & 2.0 & 3.0 & 4.0 & 5.0 & 6.0 \\
 \midrule
AFC\,{\scriptsize (NME)} & 60.16 & 62.27& \red{62.58} & \black{62.37} & 61.83 \\
AFC\,{\scriptsize (CNN)} & 59.98 &  61.88 & \red{62.18} & \black{62.34} &  62.08 \\
 \bottomrule
\end{tabular}
}
\end{table}

\vspace{-0.2cm}
\paragraph{Effect of $\lambda_\text{disc}$ for discrepancy loss}
\label{appendix:effect_lambda_ell}
Table~\ref{tab:ablation} presents the results on CIFAR100 with 50 stages by varying $\lambda_\text{disc}$, which controls the balance between the previous knowledge and the current task.
As shown in Table~\ref{tab:ablation}, the proposed method is not sensitive to the hyperparameter.

%


\section{Conclusion}
\label{sec:conclusion}
We presented a simple but effective knowledge distillation approach via adaptive feature consolidation, which adjusts the weight of each feature map for balancing between adaptivity to new data and robustness to old ones in class incremental learning scenarios.
We formulated the importance of a feature in a principled way by deriving the relationship between the discrepancy in the distribution of each feature and the change in the loss. 
We further relaxed the original objective function and incorporated knowledge distillation to effectively reduce computational cost and memory overhead. 
Experimental results show that the proposed approach outperforms the existing ones by large margins.

{
\small
\vspace{-0.07in}
\paragraph{Acknowledgments}
This work was partly supported by Samsung Electronics Co., Ltd., the IITP grants [2021-0-01343, Artificial Intelligence Graduate School Program (Seoul National University), 2021-0-02068, Artificial Intelligence Innovation Hub], and the National Research Foundation of Korea (NRF) grant [2022R1A2C3012210] funded by the Korea government (MSIT).%

{\small
\bibliographystyle{ieee_fullname}
\bibliography{egbib}

\begin{thebibliography}{10}\itemsep=-1pt

\bibitem{abati2020conditional}
Davide Abati, Jakub Tomczak, Tijmen Blankevoort, Simone Calderara, Rita
  Cucchiara, and Babak~Ehteshami Bejnordi.
\newblock {Conditional Channel Gated Networks for Task-Aware Continual
  Learning}.
\newblock In {\em CVPR}, 2020.

\bibitem{ahn2021ss}
Hongjoon Ahn, Jihwan Kwak, Subin Lim, Hyeonsu Bang, Hyojun Kim, and Taesup
  Moon.
\newblock {SS-IL: Separated Softmax for Incremental Learning}.
\newblock In {\em ICCV}, 2021.

\bibitem{aljundi2018memory}
Rahaf Aljundi, Francesca Babiloni, Mohamed Elhoseiny, Marcus Rohrbach, and
  Tinne Tuytelaars.
\newblock {Memory Aware Synapses: Learning what (not) to forget}.
\newblock In {\em ECCV}, 2018.

\bibitem{aljundi2019gradient}
Rahaf Aljundi, Min Lin, Baptiste Goujaud, and Yoshua Bengio.
\newblock {Gradient Based Sample Selection for Online Continual Learning}.
\newblock In {\em NeurIPS}, 2019.

\bibitem{amodei2016deep}
Dario Amodei, Sundaram Ananthanarayanan, Rishita Anubhai, Jingliang Bai, Eric
  Battenberg, Carl Case, Jared Casper, Bryan Catanzaro, Qiang Cheng, Guoliang
  Chen, et~al.
\newblock {Deep Speech 2: End-to-End Speech Recognition in English and
  Mandarin}.
\newblock In {\em ICML}, 2016.

\bibitem{angermueller2016deep}
Christof Angermueller, Tanel P{\"a}rnamaa, Leopold Parts, and Oliver Stegle.
\newblock {Deep Learning for Computational Biology}.
\newblock {\em Molecular Systems Biology}, 2016.

\bibitem{ba2014deep}
Jimmy Ba and Rich Caruana.
\newblock {Do Deep Nets really Need to be Deep?}
\newblock In {\em NIPS}, 2014.

\bibitem{benjamin2019measuring}
Ari Benjamin, David Rolnick, and Konrad Kording.
\newblock {Measuring and Regularizing Networks in Function Space}.
\newblock In {\em ICLR}, 2019.

\bibitem{castro2018end}
Francisco~M Castro, Manuel~J Mar{\'\i}n-Jim{\'e}nez, Nicol{\'a}s Guil, Cordelia
  Schmid, and Karteek Alahari.
\newblock {End-to-End Incremental Learning}.
\newblock In {\em ECCV}, 2018.

\bibitem{chan2016listen}
William Chan, Navdeep Jaitly, Quoc Le, and Oriol Vinyals.
\newblock {Listen, Attend and Spell: A Neural Network for Large Vocabulary
  Conversational Speech Recognition}.
\newblock In {\em ICASSP}, 2016.

\bibitem{dhar2019learning}
Prithviraj Dhar, Rajat~Vikram Singh, Kuan-Chuan Peng, Ziyan Wu, and Rama
  Chellappa.
\newblock {Learning without Memorizing}.
\newblock In {\em {CVPR}}, 2019.

\bibitem{douillard2020podnet}
Arthur Douillard, Matthieu Cord, Charles Ollion, and Thomas Robert.
\newblock {PODNet: Pooled Outputs Distillation for Small-Tasks Incremental
  Learning}.
\newblock In {\em ECCV}, 2020.

\bibitem{he2016deep}
Kaiming He, Xiangyu Zhang, Shaoqing Ren, and Jian Sun.
\newblock {Deep Residual Learning for Image Recognition}.
\newblock In {\em CVPR}, 2016.

\bibitem{hinton2015distilling}
Geoffrey Hinton, Oriol Vinyals, and Jeff Dean.
\newblock {Distilling the Knowledge in a Neural Network}.
\newblock {\em arXiv preprint arXiv:1503.02531}, 2015.

\bibitem{hou2019learning}
Saihui Hou, Xinyu Pan, Chen~Change Loy, Zilei Wang, and Dahua Lin.
\newblock {Learning a Unified Classifier Incrementally via Rebalancing}.
\newblock In {\em CVPR}, 2019.

\bibitem{hsu2018re}
Yen-Chang Hsu, Yen-Cheng Liu, Anita Ramasamy, and Zsolt Kira.
\newblock {Re-evaluating Continual Learning Scenarios: A Categorization and
  Case for Strong Baselines}.
\newblock {\em arXiv preprint arXiv:1810.12488}, 2018.

\bibitem{hu2021distilling}
Xinting Hu, Kaihua Tang, Chunyan Miao, Xian-Sheng Hua, and Hanwang Zhang.
\newblock {Distilling Causal Effect of Data in Class-Incremental Learning}.
\newblock In {\em CVPR}, 2021.

\bibitem{huang2017densely}
Gao Huang, Zhuang Liu, Laurens Van Der~Maaten, and Kilian~Q Weinberger.
\newblock {Densely Connected Convolutional Networks}.
\newblock In {\em CVPR}, 2017.

\bibitem{ioffe2015batch}
Sergey Ioffe and Christian Szegedy.
\newblock {Batch Normalization: Accelerating Deep Network Training by Reducing
  Internal Covariate Shift}.
\newblock In {\em ICML}, 2015.

\bibitem{johnson2016perceptual}
Justin Johnson, Alexandre Alahi, and Li Fei-Fei.
\newblock {Perceptual Losses for Real-Time Style Transfer and
  Super-Resolution}.
\newblock In {\em ECCV}, 2016.

\bibitem{jung2016less}
Heechul Jung, Jeongwoo Ju, Minju Jung, and Junmo Kim.
\newblock {Less-forgetting Learning in Deep Neural Networks}.
\newblock {\em arXiv preprint arXiv:1607.00122}, 2016.

\bibitem{kirkpatrick2017overcoming}
James Kirkpatrick, Razvan Pascanu, Neil Rabinowitz, Joel Veness, Guillaume
  Desjardins, Andrei~A Rusu, Kieran Milan, John Quan, Tiago Ramalho, Agnieszka
  Grabska-Barwinska, et~al.
\newblock {Overcoming Catastrophic Forgetting in Neural Networks}.
\newblock {\em Proceedings of the national academy of sciences}, 2017.

\bibitem{cifar09}
Alex Krizhevsky, Vinod Nair, and Geoffrey Hinton.
\newblock {Learning Multiple Layers of Features from Tiny Images}.
\newblock Technical report, 2009.

\bibitem{lee2019neural}
Soochan Lee, Junsoo Ha, Dongsu Zhang, and Gunhee Kim.
\newblock {A Neural Dirichlet Process Mixture Model for Task-Free Continual
  Learning}.
\newblock In {\em ICLR}, 2019.

\bibitem{lenz2015deep}
Ian Lenz, Honglak Lee, and Ashutosh Saxena.
\newblock {Deep Learning for Detecting Robotic Grasps}.
\newblock {\em International Journal of Robotics Research}, 2015.

\bibitem{li2017learning}
Zhizhong Li and Derek Hoiem.
\newblock {Learning without Forgetting}.
\newblock {\em TPAMI}, 2017.

\bibitem{liu2021adaptive}
Yaoyao Liu, Bernt Schiele, and Qianru Sun.
\newblock {Adaptive Aggregation Networks for Class-Incremental Learning}.
\newblock In {\em CVPR}, 2021.

\bibitem{liu2020mnemonics}
Yaoyao Liu, Yuting Su, An-An Liu, Bernt Schiele, and Qianru Sun.
\newblock {Mnemonics Training: Multi-Class Incremental Learning without
  Forgetting}.
\newblock In {\em CVPR}, 2020.

\bibitem{long2015fully}
Jonathan Long, Evan Shelhamer, and Trevor Darrell.
\newblock {Fully Convolutional Networks for Semantic Segmentation}.
\newblock In {\em CVPR}, 2015.

\bibitem{mccloskey1989catastrophic}
Michael McCloskey and Neal~J Cohen.
\newblock {Catastrophic Interference in Connectionist Networks: The sequential
  Learning Problem}.
\newblock In {\em Psychology of Learning and Motivation}. Elsevier, 1989.

\bibitem{nam2016learning}
Hyeonseob Nam and Bohyung Han.
\newblock Learning multi-domain convolutional neural networks for visual
  tracking.
\newblock In {\em CVPR}, 2016.

\bibitem{noh2015learning}
Hyeonwoo Noh, Seunghoon Hong, and Bohyung Han.
\newblock Learning deconvolution network for semantic segmentation.
\newblock In {\em ICCV}, 2015.

\bibitem{odena2017conditional}
Augustus Odena, Christopher Olah, and Jonathon Shlens.
\newblock {Conditional Image Synthesis with Auxiliary Classifier Gans}.
\newblock In {\em ICML}, 2017.

\bibitem{ostapenko2019learning}
Oleksiy Ostapenko, Mihai Puscas, Tassilo Klein, Patrick Jahnichen, and Moin
  Nabi.
\newblock {Learning to Remember: A Synaptic Plasticity Driven Framework for
  Continual Learning}.
\newblock In {\em CVPR}, 2019.

\bibitem{park2021class}
Jaeyoo Park, Minsoo Kang, and Bohyung Han.
\newblock Class-incremental learning for action recognition in videos.
\newblock In {\em ICCV}, 2021.

\bibitem{prabhu2020gdumb}
Ameya Prabhu, Philip~HS Torr, and Puneet~K Dokania.
\newblock {GDumb: A Simple Approach that Questions Our Progress in Continual
  Learning}.
\newblock In {\em ECCV}, 2020.

\bibitem{rebuffi2017icarl}
Sylvestre-Alvise Rebuffi, Alexander Kolesnikov, Georg Sperl, and Christoph~H
  Lampert.
\newblock {iCaRL: Incremental Classifier and Representation Learning}.
\newblock In {\em CVPR}, 2017.

\bibitem{redmon2016you}
Joseph Redmon, Santosh Divvala, Ross Girshick, and Ali Farhadi.
\newblock {You Only Look Once: Unified, Real-Time Object Detection}.
\newblock In {\em CVPR}, 2016.

\bibitem{romero2014fitnets}
Adriana Romero, Nicolas Ballas, Samira~Ebrahimi Kahou, Antoine Chassang, Carlo
  Gatta, and Yoshua Bengio.
\newblock {FitNets: Hints For Thin Deep Nets}.
\newblock In {\em ICLR}, 2015.

\bibitem{ILSVRC15}
Olga Russakovsky, Jia Deng, Hao Su, Jonathan Krause, Sanjeev Satheesh, Sean Ma,
  Zhiheng Huang, Andrej Karpathy, Aditya Khosla, Michael Bernstein,
  Alexander~C. Berg, and Li Fei-Fei.
\newblock {ImageNet Large Scale Visual Recognition Challenge}.
\newblock {\em IJCV}, 2015.

\bibitem{rusu2016progressive}
Andrei~A Rusu, Neil~C Rabinowitz, Guillaume Desjardins, Hubert Soyer, James
  Kirkpatrick, Koray Kavukcuoglu, Razvan Pascanu, and Raia Hadsell.
\newblock {Progressive Neural Networks}.
\newblock {\em arXiv preprint arXiv:1606.04671}, 2016.

\bibitem{selvaraju2017grad}
Ramprasaath~R Selvaraju, Michael Cogswell, Abhishek Das, Ramakrishna Vedantam,
  Devi Parikh, and Dhruv Batra.
\newblock {Grad-Cam: Visual Explanations from Deep Networks via Gradient-based
  Localization}.
\newblock In {\em ICCV}, 2017.

\bibitem{shin2017continual}
Hanul Shin, Jung~Kwon Lee, Jaehong Kim, and Jiwon Kim.
\newblock {Continual Learning with Deep Generative Replay}.
\newblock In {\em NIPS}, 2017.

\bibitem{simon2021learning}
Christian Simon, Piotr Koniusz, and Mehrtash Harandi.
\newblock {On Learning the Geodesic Path for Incremental Learning}.
\newblock In {\em CVPR}, 2021.

\bibitem{NIPS2014_a14ac55a}
Ilya Sutskever, Oriol Vinyals, and Quoc~V Le.
\newblock {Sequence to Sequence Learning with Neural Networks}.
\newblock In {\em NIPS}, 2014.

\bibitem{van2019three}
Gido~M van~de Ven and Andreas~S Tolias.
\newblock {Three Scenarios for Continual Learning}.
\newblock {\em arXiv preprint arXiv:1904.07734}, 2019.

\bibitem{vaswani2017attention}
Ashish Vaswani, Noam Shazeer, Niki Parmar, Jakob Uszkoreit, Llion Jones,
  Aidan~N Gomez, Lukasz Kaiser, and Illia Polosukhin.
\newblock {Attention is All you Need}.
\newblock In {\em NIPS}, 2017.

\bibitem{wu2019large}
Yue Wu, Yinpeng Chen, Lijuan Wang, Yuancheng Ye, Zicheng Liu, Yandong Guo, and
  Yun Fu.
\newblock {Large Scale Incremental Learning}.
\newblock In {\em CVPR}, 2019.

\bibitem{yan2021dynamically}
Shipeng Yan, Jiangwei Xie, and Xuming He.
\newblock {DER: Dynamically Expandable Representation for Class Incremental
  Learning}.
\newblock In {\em CVPR}, 2021.

\bibitem{zagoruyko2016paying}
Sergey Zagoruyko and Nikos Komodakis.
\newblock {Paying More Attention to Attention: Improving the Performance of
  Convolutional Neural Networks via Attention Transfer}.
\newblock In {\em ICLR}, 2017.

\bibitem{zenke2017continual}
Friedemann Zenke, Ben Poole, and Surya Ganguli.
\newblock {Continual Learning Through Synaptic Intelligence}.
\newblock In {\em ICML}, 2017.

\bibitem{zhao2020maintaining}
Bowen Zhao, Xi Xiao, Guojun Gan, Bin Zhang, and Shu-Tao Xia.
\newblock {Maintaining Discrimination and Fairness in Class incremental
  Learning}.
\newblock In {\em CVPR}, 2020.

\end{thebibliography}
}
\clearpage
%
\begin{table*}[!t]
\caption{Comparisons between AFC and the state-of-the-art algorithms on CIFAR100 with two strategies for maintaining exemplar sets. 
One stores 2,000 images for the entire old classes in total, where 2,000 images are equally distributed across all the classes in the previous tasks while the other always keep 20 images for each of old classes, which are referred to as $R_{\text{total}} = 2,000$ and $R_{\text{per}}=20$, respectively. 
Note that the results with $R_{\text{per}}=20$ are copied from our main paper and methods with asterisks (*) denote our reproductions using the official code given by the authors.
A bold-faced number represents the best performance in each column.}
\label{tab:strategy}
\centering
\scalebox{0.90}{
\begin{tabular}{@{}l|cc|cc@{}}
 \toprule
 & \multicolumn{4}{c}{CIFAR100}\\
 & \multicolumn{2}{c|}{$R_{\text{total}} = 2000$} & \multicolumn{2}{c}{$R_{\text{per}}=20$} \\
  & 50 stages & 10 stages & 50 stages & 10 stages \\
  \multicolumn{1}{r|}{New classes per stage} & 1 & 5 & 1 & 5\\
 \midrule
iCaRL~\cite{rebuffi2017icarl} &42.34   &56.52  & 44.20 & 53.78\\ 
 BiC~\cite{wu2019large} & 48.44 & 55.03 &47.09 & 53.21 \\
 UCIR\,{\scriptsize (NME)}~\cite{hou2019learning}  & 54.08 & 62.89 & 48.57  & 60.83 \\
 UCIR\,{\scriptsize (CNN)}~\cite{hou2019learning}& 55.20 &63.62 & 49.30 & 61.22 \\
 GDumb*~\cite{prabhu2020gdumb} & 60.98 & 61.33 & 59.76 & 60.24 \\
 PODNet\,{\scriptsize (NME)}~\cite{douillard2020podnet} & 62.47 & 64.60 & 61.40 & 64.03 \\
 PODNet\,{\scriptsize (CNN)}~\cite{douillard2020podnet} &  61.87 & 64.68 &57.98 & 63.19\\
 PODNet\,{\scriptsize (NME)}*~\cite{douillard2020podnet} & 60.53 & 64.30 & 56.78& 63.27 \\
 PODNet\,{\scriptsize (CNN)}*~\cite{douillard2020podnet} & 61.86  & 64.92 & 57.86 & 62.78 \\
 \hdashline
 AFC\,{\scriptsize (NME)}& \black{63.88} & \black{65.42} &\red{62.58} & \black{64.29} \\
 AFC \,{\scriptsize (CNN)}& \red{64.01} & \red{65.92} &\black{62.18} & \red{64.98} \\
 \bottomrule
\end{tabular}
}
\end{table*}
\section{Appendix}
\label{appendix:appendix}
This section first derives the upper bound of the expectation of the loss change, $\triangle \mathcal{L}(Z'_{\ell,c})$, in model updates, which is a detailed version of \eqref{eq:expectation_over_loss_prime_upper_bound_0}, \eqref{eq:expectation_over_loss_prime_upper_bound_1}, and \eqref{eq:expectation_over_loss_prime_upper_bound_2}.
Secondly, we present the results from two different strategies to maintain exemplar sets.
Thirdly, we compare with PODNet based on other metrics and using the random exemplar selection rule.
Finally, we discuss the limitation.
\subsection{Detailed Derivation of \eqref{eq:expectation_over_loss_prime_upper_bound_0}, \eqref{eq:expectation_over_loss_prime_upper_bound_1}, and \eqref{eq:expectation_over_loss_prime_upper_bound_2}} 
\label{appendix:detail_derivation}
By Eq.~\eqref{eq:loss_approx}, the expected loss change given by model updates is given by:
\begin{align}
    \mathbb{E}\Big[\triangle \mathcal{L}(Z'_{\ell,c})\Big]  
	= \mathbb{E}\Big[  \langle \nabla_{Z_{\ell,c}} \mathcal{L}(G_\ell(Z_\ell), y), Z'_{\ell,c} - Z_{\ell,c} \rangle _F  \Big], \nonumber	
\end{align}
where the expectations are taken over $\mathcal{P}^{t-1}_{\text{data}}$. 
Let $A_{m,n}$ be the element at the $m^{\text{th}}$ row and $n^{\text{th}}$ column of the matrix $A$.
Then, Eq.~\eqref{eq:expectation_over_loss_prime_upper_bound_1} is given by the following derivation:
\begin{align}
	 &\langle \nabla_{Z_{\ell,c}}  \mathcal{L}(G_\ell(Z_\ell), y), Z'_{\ell,c} - Z_{\ell,c} \rangle _F   \nonumber \\
	 &= \text{tr}(\nabla_{Z_{\ell,c}} \mathcal{L}(G_\ell(Z_\ell), y)^T(Z'_{\ell,c} - Z_{\ell,c}))  \nonumber  \\ 	
	&= \sum_{i=1}^{W_\ell} (\nabla_{Z_{\ell,c}} \mathcal{L}(G_\ell(Z_\ell), y)^T(Z'_{\ell,c} - Z_{\ell,c}))_{i,i} 
	 \nonumber \\
	 &= \sum_{i=1}^{W_\ell} \sum_{j=1}^{H_\ell} (\nabla_{Z_{\ell,c}} \mathcal{L}(G_\ell(Z_\ell), y)^T)_{i,j}(Z'_{\ell,c} - Z_{\ell,c})_{j,i} \nonumber \\
	 &= \sum_{i=1}^{W_\ell} \sum_{j=1}^{H_\ell} (\nabla_{Z_{\ell,c}} \mathcal{L}(G_\ell(Z_\ell), y))_{j,i} \cdot (Z'_{\ell,c} - Z_{\ell,c})_{j,i} \nonumber \\
	 &\leq \sqrt{\hspace{-0.10cm}\Big(\sum_{i=1}^{W_\ell} \sum_{j=1}^{H_\ell}(\nabla_{Z_{\ell,c}} \mathcal{L}(G_\ell(Z_\ell), y))_{j,i}^2\Big)} \times \nonumber \\
	& \quad \; \sqrt{\Big(\sum_{i=1}^{W_\ell} \sum_{j=1}^{H_\ell}(Z'_{\ell,c} - Z_{\ell,c})_{j,i}^2\Big)\hspace{-0.10cm} }\nonumber \\
	 & = \| \nabla_{Z_{\ell,c}} \mathcal{L}(G_\ell(Z_\ell), y) \|_F \cdot \| Z'_{\ell,c} - Z_{\ell,c}\| _F 
	 \label{eq:appendix7}, 
\end{align}

On the other hand, for any given random variable $X$ and $Y$, the following inequality holds for any real $k$: 
\begin{align}
&\mathbb{E}[(kX+Y)^2] \geq 0 \nonumber \\ 
 \iff  &\mathbb{E}[X^2] k^2 + 2k \mathbb{E}[XY] + \mathbb{E}[Y^2] \geq 0.\label{eq:appendix8}
\end{align}
Because the number of all zeros  for the above inequality is at most 1, the discriminant must be less than or equal to 0 as follows:
\begin{align}
 &\mathbb{E}[XY]^2 - \mathbb{E}[X^2] \mathbb{E}[Y^2] \leq 0 \nonumber  \\
 \iff |&\mathbb{E}[XY]| \leq \sqrt{\mathbb{E}[X^2] \mathbb{E}[Y^2]}.\label{eq:appendix10} 
\end{align}
By Eq.~\eqref{eq:appendix10}, we have the upper bound of the expectation of Eq.~\eqref{eq:appendix7}, which is given by  
\begin{align}
	\mathbb{E} \Big[ \| \nabla_{Z_{\ell,c}} \mathcal{L}(G_\ell(Z_\ell), y) \|_F \cdot \| Z'_{\ell,c} - Z_{\ell,c}\| _F \Big] \nonumber \\
	\leq \sqrt{\mathbb{E}\Big[ \| \nabla_{Z_{\ell,c}} \mathcal{L}(G_\ell(Z_\ell), y) \|_F ^2 \Big] \cdot \mathbb{E}\Big[\| Z'_{\ell,c} - Z_{\ell,c}\| _F ^2\Big]},
	\label{eq:appendix11} 
\end{align}
and, we obtain Eq.~\eqref{eq:expectation_over_loss_prime_upper_bound_2}.  

\subsection{Results from Two Different Strategies to Maintain Exemplar sets} 
\label{appendix:effect_memory_budget_fixed_capacity}
In all experiments of the main paper, we store the same number of exemplars for each class in the previous tasks, \eg 20 examples per class ($R_\text{per} = 20$).
We also run the experiments with another strategy, where we allocate a fixed amount of memory for the entire old tasks, \eg 2,000 examples in total ($R_\text{total} = 2000$).
Table~\ref{tab:strategy} illustrates the results on CIFAR-100 from the two strategies, which show the outstanding performance of the proposed approach, AFC, in both cases.

%
\begin{table*}[!t]
\caption{Performance comparison between AFC and PODNet on CIFAR100 based on backward transfer metric.} 
\label{tab:Cifar_backward_transfer}
\centering
\scalebox{0.90}{
\begin{tabular}{@{}l|cccc@{}}
 \toprule
 \multicolumn{1}{c|}{Backward Transfer (\%)} & \multicolumn{4}{c}{CIFAR100}\\
& 50 stages & 25 stages & 10 stages & 5 stages\\
 \multicolumn{1}{r|}{New classes per stage} & 1 & 2 & 5 & 10\\
 \midrule
  PODNet\,{\scriptsize (NME)}*~\cite{douillard2020podnet} &  -20.05\mypm{}0.85 & \black{-17.38\mypm{}1.11} & \red{-14.49\mypm{}0.78} &  \red{-11.98\mypm{}0.78}\\
 PODNet\,{\scriptsize (CNN)}*~\cite{douillard2020podnet} & -29.49\mypm{}1.65 & -27.06\mypm{}1.71 & -25.60\mypm{}1.32 & -23.45\mypm{}1.58\\
 \hdashline
 AFC\,{\scriptsize (NME)}& \red{-15.34\mypm{}0.54} & \red{-13.64\mypm{}0.97} & \black{-14.86\mypm{}0.43} & \black{-12.91\mypm{}1.90}\\
 AFC\,{\scriptsize (CNN)}& \black{-18.52\mypm{}1.45} & -17.42\mypm{}0.77  & -18.53\mypm{}0.80 & -17.18\mypm{}1.19\\
 \bottomrule
\end{tabular}
}
\end{table*}
%
%
\begin{table*}[ht!]
\caption{Performance comparison between AFC and PODNet on CIFAR100 based on average accuracy metric.} 
\label{tab:Cifar_average_accuracy}
\centering
\scalebox{0.90}{
\begin{tabular}{@{}l|cccc@{}}
 \toprule
 \multicolumn{1}{c|}{Average Accuracy (\%)} & \multicolumn{4}{c}{CIFAR100}\\
& 50 stages & 25 stages & 10 stages & 5 stages\\
 \multicolumn{1}{r|}{New classes per stage} & 1 & 2 & 5 & 10\\
 \midrule
  PODNet\,{\scriptsize (NME)}*~\cite{douillard2020podnet} &  46.80\mypm{}1.21 & 49.63\mypm{}1.19 & 53.10\mypm{}0.78 & 55.77\mypm{}0.55 \\
 PODNet\,{\scriptsize (CNN)}*~\cite{douillard2020podnet} & 48.57\mypm{}0.25 & 51.30\mypm{}0.61 & 52.70\mypm{}0.36 & 54.80\mypm{}0.70 \\
 \hdashline
 AFC\,{\scriptsize (NME)}& \black{52.30 \mypm{} 0.50} & \black{54.37 \mypm{} 0.64} & \black{54.70\mypm{}0.72} & \black{56.73\mypm{}1.07} \\
 AFC\,{\scriptsize (CNN)}& \red{52.77 \mypm{} 0.21} & \red{54.50 \mypm{} 0.44} & \red{54.93\mypm{}0.51} & \red{56.87\mypm{}0.40} \\
 \bottomrule
\end{tabular}
}
\end{table*}
%
\subsection{Results by Other Metrics}
\label{appendix:results_other_metrics}
We report the backward transfer and average accuracy by comparing AFC with PODNet. Table~\ref{tab:Cifar_backward_transfer} and~\ref{tab:Cifar_average_accuracy}  demonstrate that AFC also outperforms PODNet in terms of the metrics in most cases. Although AFC is marginally worse than PODNet (NME) on lower stage settings (10 and 5 stages) in terms of the backward transfer metric, it clearly outperforms PODNet for 50 and 25 stage settings, which suffer from more catastrophic forgetting.

 %
\begin{table*}[ht!]
\caption{Performance comparison between AFC and PODNet on CIFAR100 using the random selection for exemplars.} 
\label{tab:Cifar_random}
\centering
\scalebox{0.90}{
\begin{tabular}{@{}l|cccc@{}}
 \toprule
  \multicolumn{1}{c|}{Random Exemplar Selection Rule} & \multicolumn{4}{c}{CIFAR100}\\
  & 50 stages & 25 stages & 10 stages & 5 stages\\
 \multicolumn{1}{r|}{New classes per stage} & 1 & 2 & 5 & 10\\
 \midrule
  PODNet\,{\scriptsize (NME)}*~\cite{douillard2020podnet} & 54.99\mypm{}0.84 & 58.11\mypm{}0.79 & 61.75\mypm{} 1.04 & 63.67\mypm{}1.02\\
 PODNet\,{\scriptsize (CNN)}*~\cite{douillard2020podnet} & 55.55\mypm{}1.83 & 58.18\mypm{}1.56 & 61.12\mypm{}1.39 & 63.35\mypm{}1.01\\
 \hdashline
 AFC\,{\scriptsize (NME)}& \red{60.37\mypm{}0.83} & \red{62.12\mypm{}0.68} & \black{62.65\mypm{}0.95} & \black{64.31\mypm{}0.55}\\
 AFC\,{\scriptsize (CNN)}& \black{59.51\mypm{}1.04} & \black{61.55\mypm{}0.96} & \red{62.90\mypm{}1.15} & \red{64.53\mypm{}1.10}\\
 \bottomrule
\end{tabular}
}
\end{table*}
%
\subsection{Results by Random Exemplar Selection Rule}
\label{appendix:results_random_selection}
Table~\ref{tab:Cifar_random} shows that the nearest-mean of exemplars rule is more effective than the random selection rule for both AFC and PODNet compared with the result of Table~\ref{tab:Cifar} in the main paper. Note that AFC also outperforms PODNet combined with the random selection for the exemplars.

\subsection{Limitation}
\label{appendix:limitation}
\vspace{-0.1cm}
Existing methods for class incremental learning essentially require the exemplar sets for the old tasks in order to reduce the catastrophic forgetting problem.
Storing the exemplar sets such as the personal medical datasets potentially poses risk in privacy.
Moreover, the proposed method and other functional regularization approaches require additional forward passes for the old models, which leads to extra computational cost in terms of FLOPs, memory, and power consumptions.
Therefore, it would be important to develop algorithms applicable to the resource-hungry systems.

\end{document}